\documentclass[nohyperref]{article}

\usepackage{microtype}
\usepackage{graphicx}
\usepackage{subfigure}
\usepackage{booktabs} 

\usepackage{hyperref}

\usepackage[accepted]{icml2022}

\usepackage{amsmath}
\usepackage{amssymb}
\usepackage{mathtools}
\usepackage{amsthm}
\usepackage{xcolor}

\usepackage[capitalize,noabbrev]{cleveref}

\theoremstyle{plain}
\newtheorem{theorem}{Theorem}[section]
\newtheorem{proposition}[theorem]{Proposition}
\newtheorem{lemma}[theorem]{Lemma}
\newtheorem{corollary}[theorem]{Corollary}
\theoremstyle{definition}
\newtheorem{definition}[theorem]{Definition}
\newtheorem{assumption}[theorem]{Assumption}
\newtheorem{remark}[theorem]{Remark}
\newtheorem{property}[theorem]{Property}
\usepackage{array}
\usepackage[textsize=tiny]{todonotes}

\DeclareMathOperator{\E}{\mathbb{E}}
\newcommand{\norm}[1]{\left\lVert#1\right\rVert}
\newcommand{\bs}[1]{\boldsymbol{#1}}
\usepackage{amsmath}

\DeclareMathOperator*{\argmin}{arg\,min}

\icmltitlerunning{Probabilistic Bilevel Coreset Selection}

\begin{document}

\twocolumn[

\icmltitle{Probabilistic Bilevel Coreset Selection}



\icmlsetsymbol{equal}{*}

\begin{icmlauthorlist}
\icmlauthor{Xiao Zhou}{equal,yyy}
\icmlauthor{Renjie Pi}{equal,yyy}
\icmlauthor{Weizhong Zhang}{equal,yyy}
\icmlauthor{Yong Lin}{yyy}
\icmlauthor{Tong Zhang}{yyy,goog}
\end{icmlauthorlist}

\icmlaffiliation{yyy}{The Hong Kong University of Science and Technology}
\icmlaffiliation{goog}{Google Research}
\icmlcorrespondingauthor{Tong Zhang }{tongzhang@tongzhang-ml.org}


\icmlkeywords{Machine Learning, ICML, coreset selection, bilevel optimization, probabilistic method}

\vskip 0.3in]



\printAffiliationsAndNotice{\icmlEqualContribution} 

\begin{abstract}
The goal of coreset selection in supervised learning is to produce a weighted subset of data, so that training only on the subset achieves similar performance as training on the entire dataset. Existing methods achieved promising results in resource-constrained scenarios such as continual learning and streaming. However, most of the existing algorithms are limited to traditional machine learning models. A few algorithms that can handle large models adopt greedy search approaches due to the difficulty in solving the discrete subset selection problem, which is computationally costly when coreset becomes larger and often produces suboptimal results. 
In this work, for the first time we propose a continuous probabilistic bilevel formulation of coreset selection by learning a probablistic weight for each training sample. The overall objective is posed as a bilevel optimization problem, where 1) the inner loop samples coresets and train the model to convergence and 2) the outer loop updates the sample probability progressively according to the model's performance. Importantly, we develop an efficient solver to the bilevel optimization problem via unbiased policy gradient without trouble of implicit differentiation. We provide the convergence property of our training procedure and demonstrate the superiority of our algorithm against various coreset selection methods in various tasks, especially in more challenging label-noise and class-imbalance scenarios.
\end{abstract}

\section{Introduction\label{sec:introduction}}
In the last decade,  deep neural networks (DNNs) have achieved tremendous successes in multiple areas such as computer vision \cite{simonyan2014very,he2016deep} and natural language processing \cite{NIPS2017_3f5ee243}. These superior performances are mostly achieved via learning from huge amounts of data. However, this data-driven paradigm also poses several new challenges: 1) the cumbersome dataset becomes harder to store and transfer; 2) for some real applications, such as continual learning, one can only access a small number of training data at each stage of training; 3) in some more extreme scenarios, where the training data is incorrectly labelled, or they are collected from different domains, more training data may even hurt the model's performance.
To address these issues,  a natural idea is to select a small subset (i.e., {\it coreset}) comprised of most informative training samples, such that training on this subset can achieve comparable or even better performance with that on the full dataset, which is verified in Appendix \ref{sec:more-results}.  Therefore, how to construct a good coreset for DNNs now  becomes a crucial problem.

We notice that, coreset selection has been investigated for the traditional machine learning models, e.g., SVM \cite{tsang2005core}, logistic regression \cite{huggins2016coresets} and Gaussian mixture model 
\cite{lucic2017training}, for a long time to accelerate the training process  and lots of effective methods have been developed. The idea of these studies is to find a small weighted training subset, whose objective function is close to the one of the full training set at any point in the parameter space. A typical formulation is uniform function approximation, that is to find a small subset $\hat{\mathcal{D}}$ together with the non-negative weights $\hat{\mathcal{W}}=\{w_i: i \in \hat{\mathcal{D}}\}$ from the training dataset $\mathcal{D}$ satisfying that 
\begin{equation}
|\mathcal{L}(\bs{\theta})-\hat{\mathcal{L}}(\bs{\theta})| \leq \epsilon \mathcal{L}(\bs{\theta}), \mbox{ for any } \bs{\theta} \in \mathbb{R}^p, \label{eqn:coreset-shallow-model}
\end{equation}
where $\mathcal{L}(\bs{\theta}) = \frac{1}{|\mathcal{D}|}\sum_{i\in \mathcal{D}} \ell(\bs{\theta}; \mathbf{x}_i, \mathbf{y}_i)$ is the objective function on the full dataset and $\hat{\mathcal{L}}(\bs{\theta}) = \frac{1}{|\hat{\mathcal{D}}|}\sum_{i\in \hat{\mathcal{D}}} w_i \ell(\bs{\theta}; \mathbf{x}_i, \mathbf{y}_i)$ is the one on the selected subset $\hat{\mathcal{D}}$, $\epsilon$ is a small number to control the approximation error. When $\epsilon$ is small enough, the comparable performance of the model learned from $\hat{\mathcal{L}}$ can be guaranteed. Nevertheless, it has been shown that these methods cannot be applied to  DNNs directly \cite{borsos2020coresets}. The reason is that as DNNs are always highly nonconvex and the hypothesis set is significantly larger than traditional models, to obtain small uniform approximation error $\epsilon$, one has to select a very large coreset $\hat{\mathcal{D}}$. 

More recently, bilevel optimization \cite{borsos2020coresets} has been introduced to construct the coreset for DNNs.  Their motivation is that the only thing we really care about is the performance of the model trained on
the coreset, i.e., the optimum of $\hat{\mathcal{L}}$, instead of achieving small approximation error for the loss function  in the {\it whole} parameter space. Therefore,  the  bilevel optimization with a cardinality constraint on $\hat{\mathcal{D}}$ presented below would be a more reasonable framework for constructing the coreset:
\begin{gather}
    \min_{\hat{\mathcal{W}}, \hat{\mathcal{D}}\subset \mathcal{D}, |\hat{\mathcal{D}}|\leq K} \mathcal{L}(\bs{\theta}^*(\mathcal{W}, \hat{\mathcal{D}}))\label{eqn:discrete-weighted-bilevel-opt}\\
     s.t.~~ \bs{\theta}^*(\hat{\mathcal{W}}, \hat{\mathcal{D}}) = \arg \min _{\bs{\theta}} \hat{\mathcal{L}}(\bs{\theta}).\nonumber
\end{gather}
However, despite some promising empirical results are reported, we notice that existing methods adopt greedy strategies to update $\hat{\mathcal{D}}$ as the outer loop involves optimization on the discrete variable set $\hat{\mathcal{D}}$, which is NP hard. For example, \cite{borsos2020coresets} starts from several randomly sampled training data and then sequentially adds sample with the largest gradient with respect to the outer objective into $\hat{\mathcal{D}}$ in each iteration. Such greedy methods generally lead to suboptimal performance and the selected coreset would be unnecessarily large.
The drawback becomes more prominent if the initially sampled data has bad quality, which may misguide the selection, especially when the dataset contains noisy labels or is class-imbalanced.

In this paper, to address the above issue, we propose a probabilistic coreset selection method based on bilevel optimization for DNNs. In contrast with the greedy methods, our key idea is to continualize the discrete bilevel optimization problem  above by probabilistic reparameterization, making the gradient-based optimization possible. In this way, we can explore the entire dataset and progressively improve the quality of the coreset during training. To be precise, we first assign each training data $i$ with a binary mask $m_i$ to indicate whether data $i$ is included in the coreset or not. Then, to continualize the problem, we parameterize $m_i$ to be a Bernoulli variable with the probability $s_i$ to be 1 and $1-s_i$ to be 0.  Thus, the cardinality of the coreset can be roughly controlled by the sum of $s_i$ and coreset selection is transformed into the problem of learning these probabilitis $s_i$. Therefore, we formulate coreset selection as a continuous bilevel optimization problem with a sparsity constraint (Eqn. (\ref{cont_formulation})). In the inner loop, we sample a coreset according to  the probability $s_i$ and use it to train the model (see Eqn. (\ref{inner_loop})). In the outer loop, we minimize the loss of the learned model on the full dataset by adjusting the sample probabilities. We develop an efficient optimization algorithm (Algorithm \ref{alg:SST}) utilizing unbiased policy gradient estimator, which calculates the probabilities' gradients via only forward instead of backward propagation and avoids the complicate calculation of implicit differentiation.

Notably, our proposed method has the following advantages: 

1) Our algorithm gradually improves the quality of the selected subset by explores the training set globally unlike greedy methods which make urgent decisions at early stage and cannot remove redundant data once added to the coreset.

2) Our algorithm obtains much more competence in the setting with label noise or class imbalance without the bother of making decisions on adding detrimental or redundant data at early stage. In contrast, the misselected data may mislead the future selection for greedy methods.

3) Our algorithm develops an efficient policy gradient solver to the bilevel optimization problem without cumbersome implicit gradient calculations. 



Moreover, we provide  the property of our training algorithm in convergence. We demonstrate the superiority of our method through extensive experimental results on various tasks, including data summarization, continual learning, streaming and feature selection and surpass state-of-the-art methods by a large margin.

Our main contributions can be  summarized as follows:
\begin{itemize}
\item To the best of our knowledge, our method is the first global coreset selection method for DNNs, where we propose a novel continualized bilevel optimization formulation and develop an efficient policy gradient solver without implicit gradient calculation.

\item We provide the convergence property showing that  our optimization method can converge similarly with  the standard nonconvex projected
stochastic gradient descent algorithm. 
\item We empirically demonstrate the superiority of our method through various experiments including data summarization, continual learning, streaming, feature selection, especically in more challenging senarios with noisy labels and class imbalance.
\end{itemize}

\section{Related Works}
\subsection{Coreset Selection}
Coreset selection aims to solve the problem of finding the most informative subset from the full set, which can be used to solve the optimization problem and obtain similar performance as the full set. Several coreset selection methods are designed for specific learning algorithms, such as K-means \cite{feldman2011unified, har2007smaller}, SVM \cite{tsang2005core}, logistic regression \cite{huggins2016coresets} and Gaussian mixture model \cite{lucic2017training, feldman2011scalable}. These methods only work for traditional models and can not be directly adopted for DNNs. \cite{borsos2020coresets} proposed a bilevel optimization framework for coreset selection for DNN. However, due to the difficulty in solving the discrete optimization problem, they adopt greedy search algorithm to sequentially select new samples. Since the samples selected in the early stage can never be removed afterwards, their method generally results in suboptimal performance and deteriorates severely in challenging scenarios, where the initial samples can have bad quality and misguide the search. Moreover, the greedy-based algorithm faces expensive computational cost due to its demand in solving a bilevel optimization problem for every sample added to the coreset. Our method searches the coreset globally and does not suffer from those drawbacks.



\vspace{-.2cm}
\subsection{Continual Learning and Streaming}
Continual learning (CL) \cite{kirkpatrick2017overcoming, lopez2017gradient,rebuffi2017icarl} aims to tackle the scenario where a series of different tasks are learnt sequentially using the same model. In this work, we mainly focus on replay-based CL methods, which keeps a constant number of data of previous tasks to alleviate the catastrophic forgetting problem. Streaming \cite{aljundi2019gradient,hayes2019memory,chrysakis2020online} is more challenging in the sense that it does not have the concept of tasks and data is sequentially given to the model. In these cases, coresets can be adopted to construct the replay memory to choose the informative data which well represents each task.

\vspace{-.2cm}
\subsection{Feature Selection} Feature selection (FS) \cite{cai2018feature, li2017feature, miao2016survey} aims to select a subset of important features to represent the original data, which reduces the computation and storage cost. The majority of existing works of feature selection are mainly focused on traditional models  \cite{gunecs2010multi, sulaiman2015feature, radovic2017minimum}. However, to the best of our knowledge, the exploration in this direction for deep neural networks is limited. We note that feature selection can also be considered as an instance level coreset selection task which can be naturally addressed by our proposed framework, where the selected features can be viewed as the coreset.

\vspace{-.2cm}
\subsection{Dataset Distillation}
An alternative approach for dataset compression is dataset distillation \cite{wang2018dataset,nguyen2020dataset}, which is inspired by knowledge distillation \cite{hinton2015distilling, gou2021knowledge, yao2021g}. Instead of distilling knowledge from the model parameters, dataset distillation learns a few synthetic data points for each class.
These methods work well when the model for deployment is the same as the one used for learning the synthetic data. However, as the learnt data points also encode information of the model's architecture and initialization weights, 
their performances drop significantly when the synthetic data learned on one model is used to train another model. In contrast, coreset selection is not sensitive to the model, since we do not alter the data directly.

\subsection{Bilevel Optimization}
Bilevel optimization \cite{sinha2017review} has garnered a lot of attention in recent years due to its ability to handle hierarchical decision making processes. Previous works utilize bilevel optimization in multiple areas of research, such as hyper-paramter optimization \cite{lorraine2020optimizing, maclaurin2015gradient,pedregosa2016hyperparameter, mackay2019self, franceschi2017forward,vicol2021unbiased}, meta learning \cite{finn2017model, nichol2018reptile}, neural architecture search \cite{pham2018efficient, liu2018darts, pham2018efficient, shi2020bridging, yao2021joint, gao2022unison, gao2021autobert, shi2021sparsebert} and sample re-weighting \cite{ren2018learning, shu2019meta, wang2022training, maple}. Prior to our work, \cite{borsos2020coresets} formulates coreset selection into a bi-level optimization problem and solves it using a greedy algorithm.















\section{Probabilistic Bilevel Coreset Selection}
In this section, we first present our probabilistic coreset selection framework in Section \ref{sec:framework} and then develop an efficient  training method for this framework in Section \ref{sec:optimization}. 

\subsection{Bilevel Framework for Coreset Selection}\label{sec:framework}

Consider a neural network $f(\bs{x}; \bs{\theta})$ with $\bs{\theta}$ being the trainable  parameters and $\mathcal{D}=\{(\mathbf{x}_i, \mathbf{y}_i)\}_{i=1}^{n}$ is the training dataset, we first formulate coreset selection into the following discrete bilevel optimization paradigm:
{\small
\begin{gather}
\min_{\bs{m} \in \tilde{\mathcal{C}}}\tilde{\Phi}(\bs{m})= \mathcal{L}(\boldsymbol{\theta}^{*}(\bs{m})) = \frac{1}{n} \sum_{i=1}^{n} \ell(f(\mathbf{x}_i; \boldsymbol{\theta}^{*}(\bs{m})), \mathbf{y}_i) , \label{ori_formulation}\\
s.t. ~\boldsymbol{\theta}^{*}(\bs{m}) \in \argmin_{\boldsymbol{\theta}} \hat{\mathcal{L}}(\bs{\theta}; \bs{m}) = \frac{1}{K}\sum_{i=1}^{n} m_i\ell(f(\mathbf{x}_i; \boldsymbol{\theta}), \mathbf{y}_i),\nonumber
\end{gather}}
where the mask $\bs{m}\in \{0,1\}^n$ is a binary vector with $m_i=1$ indicating   sample $i$ is selected into the coreset and otherwise excluded. $K$ is a positive integer controlling the coreset size and $\tilde{\mathcal{C}} = \{\bs{m}: m_i = 0 ~\text{or} ~1, \norm{\bs{m}}_0 \leq K\}$ is the feasible region of $\bs{m}$. Intuitively, the inner loop trains the network to converge on the selected coreset to obtain the model $\boldsymbol{\theta}^{*}\left(\bs{m}\right)$. The outer loop evaluates the loss of $\boldsymbol{\theta}^{*}\left(\bs{m}\right)$ on the full set and optimizes it to guide the learning of $\bs{m}$. 


\begin{remark}The difference between our discrete bilevel formulation  (\ref{ori_formulation}) and the existing one in (\ref{eqn:discrete-weighted-bilevel-opt}) is that our formulation has no weight $w_i$ for each sample in the coreset.  We remove these weights for two considerations: 1) our empirical results show that we can achieve good performance without weighting the coreset; 2) it enables us to develop extremely efficient training algorithm (see Section \ref{exp:time_complexity}).
\end{remark}

Noticing that the discrete nature of the mask $\bs{m}$ makes directly solving the above bilevel optimization problem intractable, we now turn to continualize it by probabilistic reparameterization, making gradient based optimization method possible. Our main idea is to view each mask $m_i$ as an independent  binary random variable and transform the problem (\ref{ori_formulation}) from optimizing in the discrete vector space into the probability space, which is continuous. Specifically, we reparameterize $m_i$ as a Bernoulli random variable with probability $s_i$ to be $1$ and $1-s_i$ to be $0$, that is $m_i \sim \operatorname{Bern}(s_i)$, where $s_i \in [0,1]$. Assuming the variables $m_i$ are independent, then we can get the distribution function of $\boldsymbol{m}$, i.e.,  $p(\boldsymbol{m}|\boldsymbol{s}) = \Pi_{i=1}^{n} (s_i)^{m_i}(1-s_i)^{(1-m_i)}$. Thus, we can control the 
coreset size via the sum of the probabilities $s_i$, i.e.,  $\boldsymbol{1}^\top\boldsymbol{s}$, since $\mathbb{E}_{\boldsymbol{m}\sim p(\boldsymbol{m}|\boldsymbol{s})}\|\boldsymbol{m}\|_0 = \sum_{i=1}^n s_i$. Therefore, $\tilde{\mathcal{C}}$ can be approximately transformed into $\mathcal{C} = \{\bs{s}: 0 \preceq \bs{s} \preceq 1, \norm{\bs{s}}_1 \leq K\}$. Finally, problem (\ref{ori_formulation}) can be naturally  relaxed  into the following excepted loss minimization problem: 
\begin{gather}
\min_{\bs{s}\in \mathcal{C}} \displaystyle ~\Phi(\bs{s})=\mathbb{E}_{ p(\boldsymbol{m}|\boldsymbol{s})} ~ \mathcal{L}(\boldsymbol{\theta}^{*}(\bs{m})), \label{cont_formulation}\\
s.t. ~\boldsymbol{\theta}^{*}(\bs{m}) \in \argmin_{\boldsymbol{\theta}} \hat{\mathcal{L}}(\bs{\theta}; \bs{m})\label{inner_loop}
\end{gather}
where $\mathcal{C} = \{\bs{s}: 0 \preceq \bs{s} \preceq 1, \norm{\bs{s}}_1 \leq K\}$ is the domain.

Some appealing features of our Formulation (\ref{cont_formulation}) are:
\begin{itemize}
    \item Our formulation is a tight relaxation (although not equivalent) of Problem (\ref{ori_formulation}). The reasons are:
    \begin{itemize}
    \item It is easy to know that $\min_{\bs{s}\in \mathcal{C}} \Phi(\bs{s})\leq \min_{\bs{m}\in \tilde{\mathcal{C}}}\tilde{\Phi}(\bs{m})$ as any deterministic binary mask $\bs{m}$ can be represented as a particular stochastic one by letting $\bs{s}_i$ be either $0$ or $1$.
    \item Our constraint $\mathcal{C}$ induces sparsity on $\bs{s}$ due to the $\ell_1$-norm and the range $[0,1]$, making most components of the optimal $\bs{s}$ either $0$ or $1$. That is, our finally learned stochastic mask is  nearly deterministic, which will be empirically verified in Section \ref{exp:convergence}. 
    \end{itemize}
    \item Due to our sparsity constraint, the selected coreset size of the inner loop, i.e., $\|\bs{m}\|$ is always small, which makes the optimization of $\theta^*$ very efficient.
    \item As shown in Eqn.(\ref{eqn:ori-PGE}), our outer objective is $\Phi(\bs{s})$ is differentiable, allowing us to use general gradient based methods for optimization. 
    \end{itemize}

\subsection{Optimization} \label{sec:optimization}
Existing bilevel optimization algorithms \cite{pedregosa2016hyperparameter,grazzi2020iteration, grazzi2021convergence} are often computationally costly due to the expensive implicit differentiation in their chain-rule based gradient estimator. To be precise, if applied to our problem, they generally estimate the gradient in the form of 
\begin{align*}
   \nabla_{\bs{s}}\Phi(\bs{s}) \approx \nabla_{\bs{s}} \bs{\theta}^*(\bs{m}) \nabla_{\bs{\theta}} \mathcal{L}(\bs{\theta}^*(\bs{m})).
\end{align*}
Hence, they need to compute the implicit differentiation of the inner loop optimum, i.e, $\nabla_{\bs{s}}\bs{\theta}^*(\bs{m})$,  which is  expensive since they have to compute the inverse of a huge hessian matrix or unroll the backward propagation for multiple steps. 

Even though some efficient bilevel optimization algorithms have been proposed to alleviate the computational burden, for instance, \cite{lorraine2020optimizing} adopted Neumann series to approximate the hessian inverse, the approximation still requires much time and leads to inefficiency.

Thanks to our probabilistic formulation of the bilevel problem, we are able to avoid these expensive computations by using Policy Gradient Estimator (PGE), which calculates the gradient using forward instead of backward propagation. Our key idea can be illustrated by the following equations:
\begin{align}
    \nabla_{\bs{s}}\Phi( \bs{s}) 
    =& \nabla_{\bs{s}} \E_{p(\boldsymbol{m}|\boldsymbol{s})} \mathcal{L}\left(\boldsymbol{\theta}^{*}( \bs{m})\right)\nonumber \\
    =& \nabla_{\bs{s}} \int \mathcal{L}\left(\boldsymbol{\theta}^{*}( \bs{m})\right) p(\boldsymbol{m}|\boldsymbol{s}) d\bs{m}\nonumber\\
    =&\int \mathcal{L}\left(\boldsymbol{\theta}^{*}( \bs{m})\right) \frac{\nabla_{\bs{s}} p(\boldsymbol{m}|\boldsymbol{s})}{p(\boldsymbol{m}|\boldsymbol{s})}p(\boldsymbol{m}|\boldsymbol{s}) d\bs{m}\nonumber \\
    =&\int \mathcal{L}\left(\boldsymbol{\theta}^{*}(\bs{m})\right) \nabla_{\bs{s}}\ln  p(\boldsymbol{m}|\boldsymbol{s})p(\boldsymbol{m}|\boldsymbol{s}) d\bs{m}\nonumber \\
    =& \mathbb{E}_{p(\boldsymbol{m}|\boldsymbol{s})} \mathcal{L}\left(\boldsymbol{\theta}^{*}( \bs{m})\right) \nabla_{\bs{s}}\ln  p(\boldsymbol{m}|\boldsymbol{s}). \label{eqn:ori-PGE}
\end{align}
It shows that $\mathcal{L}\left(\boldsymbol{\theta}^{*}( \bs{m})\right) \nabla_{\bs{s}}\ln  p(\boldsymbol{m}|\boldsymbol{s})$ is an unbiased stochastic gradient of $\nabla_{\bs{s}}\Phi(\bs{s})$, which is called policy gradient. Therefore, given the inner loop optimum $\bs{\theta}^{*}(\bs{m})$, we can update $\bs{s}$ by projected stochastic gradient descent: 
\begin{align}
        \bs{s} \leftarrow  \mathcal{P}_{\mathcal{C}}\left(\bs{s} - \eta \mathcal{L}\left(\boldsymbol{\theta}^{*}( \bs{m})\right) \nabla_{\bs{s}}\ln  p(\boldsymbol{m}|\boldsymbol{s}) \right). \tag{PGE} \label{eqn:PGE}
\end{align}
It is clear that \ref{eqn:PGE}  does not involve any implicit differentiation and its component $\mathcal{L}\left(\boldsymbol{\theta}^{*}( \bs{m})\right)$ can be computed via forward propagation. Moreover, $\ln  p(\boldsymbol{m}|\boldsymbol{s})$ has a very simple form and this projection has a closed form solution (given in the appendix \ref{sec:proj}) since the constraint $\mathcal{C}$ is quite simple. Therefore, we can update $\bs{s}$ via \ref{eqn:PGE} very efficiently. 

\begin{remark}
As we mentioned in Section \ref{sec:framework}, we remove the weights from the original framework (\ref{eqn:discrete-weighted-bilevel-opt}) because: 1) we empirically find that training using the coreset with binary weights can already achieve competitive performance, and 2) if the weights are not removed, then $\bs{\theta}(\bs{m})$ would be $\bs{\theta}(\bs{w}, \bs{m})$ and we have to compute implicit differentiation to get $\nabla_{\bs{w}} \bs{\theta}^*(\bs{w}, \bs{m})$, since the gradient of deterministic variable cannot be estimated via PGE.
\end{remark}

\begin{algorithm}[htb!]
\caption{Probabilistic Bilevel Coreset Selection }
\label{alg:SST}
\begin{algorithmic}[1]
\REQUIRE a network $\bs{\theta}$, dataset $\mathcal{D}$ and coreset size $K$.
\STATE Initialize probabilities $\bs{s}^1=\frac{K}{|\mathcal{D}|} \mathbf{1}$.
\FOR{training iteration $t = 1, 2 \ldots T$}
\STATE Sample mask $\bs{m}$ according to the probability $\bs{s}^1$. 
\STATE Train the inner loop to converge satisfies: 
\[
\boldsymbol{\theta}^{*}(\bs{m}) \leftarrow \argmin_{\boldsymbol{\theta}} \hat{\mathcal{L}}(\bs{\theta};  \bs{m})
\]
\STATE Sample a mini-batch of data: 
\[
\mathcal{B} = \left\{\left(\mathbf{x}_{1}, \mathbf{y}_{1}\right), \ldots,\left(\mathbf{x}_{B}, \mathbf{y}_{B}\right)\right\}
\]
\STATE Update $\bs{s}$ using PGE based on $\boldsymbol{\theta}^{*}( \bs{m})$ and $\mathcal{B}$: 
\[
\bs{s}^{t+1} \leftarrow  \mathcal{P}_{\mathcal{C}}\left(\bs{s}^{t} - \eta \mathcal{L}_{\mathcal{B}}(\bs{\theta}^*(\bs{m})) \nabla_{\bs{s}}\ln  p(\boldsymbol{m}|\boldsymbol{s}^t) \right)
\]
\ENDFOR

\OUTPUT The coreset $\{(\mathbf{x}_i, \mathbf{y}_i): m_i \neq 0 \mbox{ and } {(\mathbf{x}_i, \mathbf{y}_i)} \in \mathcal{D}\}$ with $\bs{m}$ sampled from $p(\boldsymbol{m}|\boldsymbol{s}^{T+1})$.
\end{algorithmic}
\end{algorithm}

Hence, we can solve our bilevel optimization problem (\ref{cont_formulation}) by alternatively: 1) sampling a mask $\bs{m}$, i.e., a coreset, from $p(\bs{m}|\bs{s})$ for the inner loop and train the model on this coreset to get $\bs{\theta}^*(\bs{m})$; 2) updating the probability $\bs{s}$ using \ref{eqn:PGE}. The detailed steps are given in Algorithm \ref{alg:SST}. Notably, our algorithm has the following advantages: 
\begin{itemize}
   \item While the greedy methods can never remove the redundant data once they are added to the coreset, our algorithm behaves like a process of sampling coreset with replacement (Step 4), where the quality of the selected subset is evaluated on the outer objective and the sampling probablity $\bs{s}$ is adjusted accordingly (Step 6) to progressively improve the coreset quality. This enables us to explore the training set more globally.
    \item Due to our sparsity constraint $\mathcal{C}$, most of the probabilities $s_i$ would automatically converge to either 0 or 1 during optimization, thus the uncertainty of the obtained coreset can be finally reduced to nearly 0. This is empirically verified in Section \ref{exp:convergence}. 
    \item The superiority of our method is more prominent in more challenging tasks involving data with corrupted labels and class imbalance, which are shown in the Section \ref{sec:experiments}. The reason is that in those scenarios, more explorations are needed to gain a global view of the entire set before deciding which samples to be added to the coreset, whereas greedy algorithms have to start making decisions in the early stage without enough knowledge (demonstrated in Section \ref{exp:noise_analysis}). 
    \item As we discussed above, our method is computationally efficient because 1) for the outer loop, PGE enables updating the probability without computing any implicit differentiation; 2) for the inner loop , the selected coreset size is always small, which makes the derivation of $\theta^*$ efficient; and 3) as opposed to the greedy algorithm, the running time of our method does not increase rapidly with the coreset size, since the number of outer updates is fixed for all coreset sizes.
\end{itemize}

The property below shows that if we solve the inner loop problem to convergence, our training algorithm can converge similarly with the standard nonconvex projected stochastic gradient descent algorithms \cite{ghadimi2016mini}. 
\begin{property}\label{thm:convergence} [Informal] Under the mild assumptions on $\Phi(\bs{s})$ and the step size $\eta$, then the average of the  expectation of the  gradient mapping norm, i.e., $$\|\frac{1}{\eta}\left(\bs{s}^{t}-\mathcal{P}_{\mathcal{C}}(\bs{s}^{t}-\eta \nabla_{\bs{s}}\Phi(\bs{s}^{t}))\right)\|_2,$$ can converge to a small value as $T\rightarrow \infty$.
\end{property}

\section{Experiments}\label{sec:experiments}

We conduct the following experiments in common application scenarios of coreset selection: 1) data summarization, where the selected coreset is directly used to train the model; 2) continual learning \cite{kirkpatrick2017overcoming, lopez2017gradient,rebuffi2017icarl} and streaming \cite{aljundi2019gradient,hayes2019memory,chrysakis2020online}, where coresets are selected from training data to construct the replay memory and resist catastrophic forgetting after sequentially learning a series of tasks; 3) feature selection \cite{cai2018feature, li2017feature, miao2016survey}, where only a subset of features are selected for training and inference.

In reality, the quality of training data can not be guaranteed. For instance, the data for each task of continual learning may not be balanced or even mislabelled. To test the effectiveness of our method on these challenging scenarios, for each of those applications mentioned above, we create more difficult settings by imposing label noise and class imbalance in the training data. We observe that other approaches including the greedy coreset method \cite{borsos2020coresets} fails significantly in such settings, while our method can still discover promising coresets, which is credited to the learnt global information before constructing the final coreset.

\subsection{Data Summarization}\label{sec:coreset}
\begin{figure*}[ht!]
	\centering
    \vspace{-2mm}
	\includegraphics[width=1.0\textwidth]{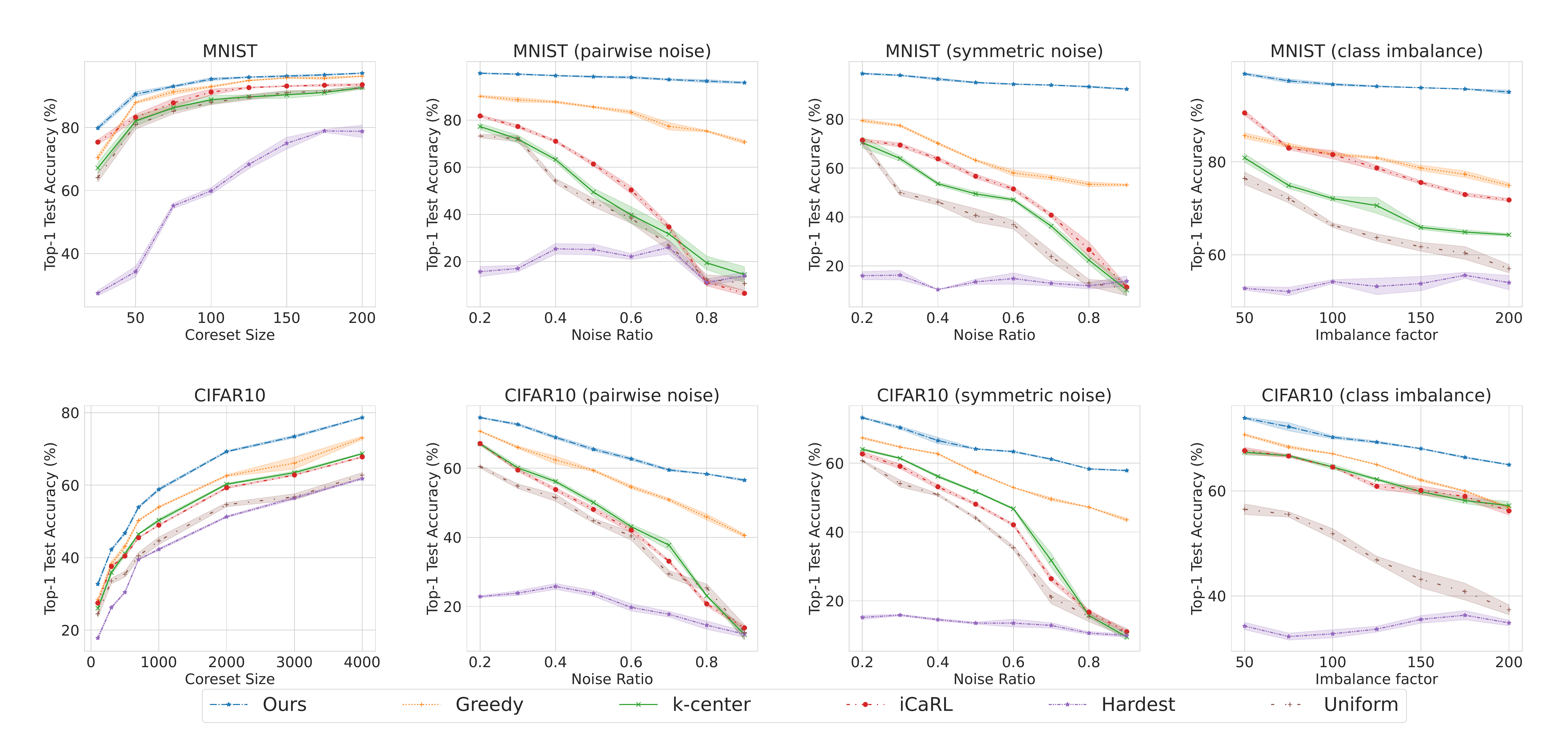}
    \vspace{-8mm}
    \caption{Performance comparison between our method and other baselines on data summarization task with various coreset sizes and different scenarios. For experiments with label noise and class imbalance, the coreset size is set to 1000 for MNIST and 5000 for CIFAR10. Our method consistently surpasses other baselines by a large margin. Notably, the performance of our method is stable even under challenging settings, while other methods begin to fail significantly.}
    \vspace{-3mm}
	\label{fig:data_summerization}
\end{figure*}

We examine the quality the coreset constructed by our method via evaluating the performance of the neural network trained on the produced coresets. Specifically, we conduct experiments on two widely used benchmarks, i.e., MNIST \cite{deng2012mnist} and CIFAR10 \cite{krizhevsky2009learning}. To make a fair comparison, we follow \cite{borsos2020coresets} to use the same model settings (shown in the Appendix \ref{sec:experimental-configuration}) when conducting the experiments.

We compare with the following competitive baselines: 1) Uniform sampling, 2) K-center clustering using the embedding from last layer \cite{sener2017active}, 3) iCaRL's selection \cite{rebuffi2017icarl}, 4) Hardest sampling \cite{aljundi2019task} and 5) Greedy coreset \cite{borsos2020coresets}. The results in the first column of Figure \ref{fig:data_summerization} demonstrate that our method consistently outperforms other baselines across various coreset sizes. Notably, our method's superiority is more prominent in the small-sized region. It is interesting that hardest sampling fails for data summarization tasks, which can be that the majority of hard samples come from only a few classes and therefore causes the selected coreset to be unbalanced.

\subsection{Data Summarization with Label Noise and Class Imbalance}\label{sec:coreset_imbalance}
We further conduct experiments in more challenging and practical scenarios, where the dataset contains corrupted labels and class-imbalanced data to showcase the effectiveness of our method. The model setting is the same as stated in Section \ref{sec:coreset} and the outer objective is calculated based on a held-out balanced validation dataset with 100 samples, comprised of 10 uniformly sampled data from each class.


For the label noise experiment, we adopt 2 types of noises: pairwise noise and symmetric noise. For the class imbalance experiment, we adopt similar setting as in \citet{cui2019class}. The detailed descriptions are given in the Appendix \ref{sec:experimental-configuration} due to space limit. The coreset sizes for the MNIST and CIFAR10 experiments are 1000 and 5000, respectively.

In these challenging settings, our algorithm's advantage over other methods becomes much more prominent, as shown in Figure \ref{fig:data_summerization}. Notably, our method is much less sensitive to the quality of entire dataset, which can be credited to the global information learnt by repeatedly sampling before constructing the final coreset. On the other hand, the greedy counterpart begins to fail dramatically as the label noise ratio and imbalance factor grow higher, which is in line with our intuition that greedy methods can be severely affected if the samples in the early phase are redundant or even detrimental. More analysis is given in Section \ref{exp:noise_analysis}.

\subsection{Continual Learning and Streaming}
Two important applications of coreset construction are continual learning \cite{kirkpatrick2017overcoming, lopez2017gradient,rebuffi2017icarl, wang2022memory,wang2022learning} and streaming \cite{aljundi2019gradient,hayes2019memory,chrysakis2020online, zou2019sufficient}. Specifically, continual learning aims to learn a series of tasks sequentially using the same model, and a constant number of data can be reserved for previous tasks to alleviate the catastrophic forgetting of early knowledge. It is thus essential to choose the most informative data for each task when constructing the replay memory. Streaming is similar to continual learning except that the data stream is not divided into tasks, which is more challenging.

\textbf{Continual Learning.} We conduct experiments on the following datasets commonly adopted by the community: 1) PermMNIST \cite{goodfellow2013empirical} constructs 10 tasks by performing different random permutation on the image pixels for each task; 2) SplitMNIST \cite{zenke2017continual} splits MNIST into five tasks, each containing two adjacent classes; and 3) SplitCIFAR10 is based on CIFAR10 and splitted in the same way as SplitMNIST. To make a fair comparison, we use the same settings as adopted in \cite{borsos2020coresets}, the details are shown in the Appendix \ref{sec:experimental-configuration}. Moreover, we design more challenging CL tasks with label noise and class imbalance to showcase the superiority of our method, where symmetric noise with 20\% noise ratio and class imbalance factor of 50 is applied to each dataset.


We compare our algorithm with data selection methods mentioned in Section \ref{sec:coreset}. The results demonstrated in Table \ref{continual_learning_table} verify that
our proposed method consistently dominates other baselines across all tasks. Our advantage is more prominent especially under challenging settings, where the performance of other methods drop significantly and our method continues to perform well. Remarkably, our method surpasses the greedy counterpart \cite{borsos2020coresets} by around 10\% on tasks with label noise.

\textbf{Streaming.} We follow \citet{borsos2020coresets} to conduct experiment in streaming setting (details are in the Appendix).
We compare our method with Reservoir sampling \cite{vitter1985random} and greedy coreset \cite{borsos2020coresets}. We also conduct experiments on label noise scenario, where symmetric label noise with 20\% noise ratio is imposed to the data. The results in Table \ref{tab:streaming} demonstrate that our method dominates other baselines on these streaming tasks.

The success on CL and streaming tasks proves that our method can select informative data that well represent each task to construct the replay memory, which performs well even under challenging conditions.

\begin{table*}[t]
\caption{Experiment result on continual learning for PermMNIST, SplitMNIST and CIFAR10 datasets. Normal stands for the standard dataset without modification; Noise represents symmetric label noise with 20\% corruption ratio and Imbalance means the data has an class imbalance factor of 50.
As demonstrated, using our proposed approach to construct the replay memory consistently surpasses other methods. Remarkably, the advantage of our global search strategy against the greedy counterpart becomes more prominent on challenging tasks: for tasks with label noise, our method surpasses the greedy counterpart by around 10\%.}
\label{continual_learning_table}
\vskip 0.15in
\begin{center}
\vspace{-2mm}
\begin{small}
\begin{tabular}
{
l m{1cm}<{\centering} m{1cm}<{\centering} m{1cm}<{\centering} m{1cm}<{\centering} m{1cm}<{\centering} m{1cm}<{\centering} m{1cm}<{\centering} m{1cm}<{\centering} m{1cm}<{\centering} }
\toprule
Datasets & \multicolumn{3}{c}{PermMNIST}&\multicolumn{3}{c}{SplitMNIST}& \multicolumn{3}{c}{SplitCIFAR-10}\\
 &Normal&Noise&Imbalance&Normal&Noise&Imbalance&Normal&Noise&Imbalance \\ \cmidrule(l){2-4}\cmidrule(l){5-7}\cmidrule(l){8-10}
Uniform&78.46&32.12&43.70&93.70&44.32&53.32& 36.20&15.80&20.50\\
k-center embeddings  & 78.57&37.53&54.53& 94.55&51.89&71.50& 36.91&16.68&21.44\\
Hardest samples    & 76.79&15.38&38.72& 91.57&17.23&56.39& 28.10&9.63&15.63\\
iCaRL    & 79.68 &39.36&67.53& 95.13&60.99&72.23& 34.52&18.47&25.47\\
Greedy Coreset      & 79.26&65.84&68.67& 96.50&81.75&84.58&  37.60&24.23&31.28  \\
Ours    & \textbf{80.60} &\textbf{74.26}&\textbf{75.32}&\textbf{98.15}&\textbf{92.23}&\textbf{94.30}& \textbf{39.10}&\textbf{31.20}&\textbf{35.30}\\
\bottomrule
\end{tabular}
\end{small}
\end{center}
\vspace{-1mm}
\vskip -0.2in
\end{table*}
\begin{table}[t]
\caption{Experiment result on Streaming for MNIST with and without label noise. Using our proposed approach to construct the replay memory consistently surpasses other methods significantly.}
\label{tab:streaming}
\begin{center}
\begin{small}
\begin{tabular}
{
l m{0.5cm}<{\centering} m{0.5cm}<{\centering} m{0.5cm}<{\centering} m{0.5cm}<{\centering}
}
\toprule
Datasets&\multicolumn{2}{c}{PermMNIST}&\multicolumn{2}{c}{SplitMNIST}\\
&Normal&Noise&Normal&Noise\\\cmidrule(l){2-3}\cmidrule(l){4-5}
Reservoir sampling & 73.21 &25.82& 90.72&22.03\\
Greedy Coreset      & 74.44&61.30& 92.59&82.52\\
Ours    & \textbf{75.5}&\textbf{70.33}& \textbf{94.20}&\textbf{90.69}\\
\bottomrule
\end{tabular}
\end{small}
\end{center}
\end{table}
\subsection{Feature Selection}
\begin{figure}[ht!]
	\centering
	\includegraphics[width=0.45\textwidth]{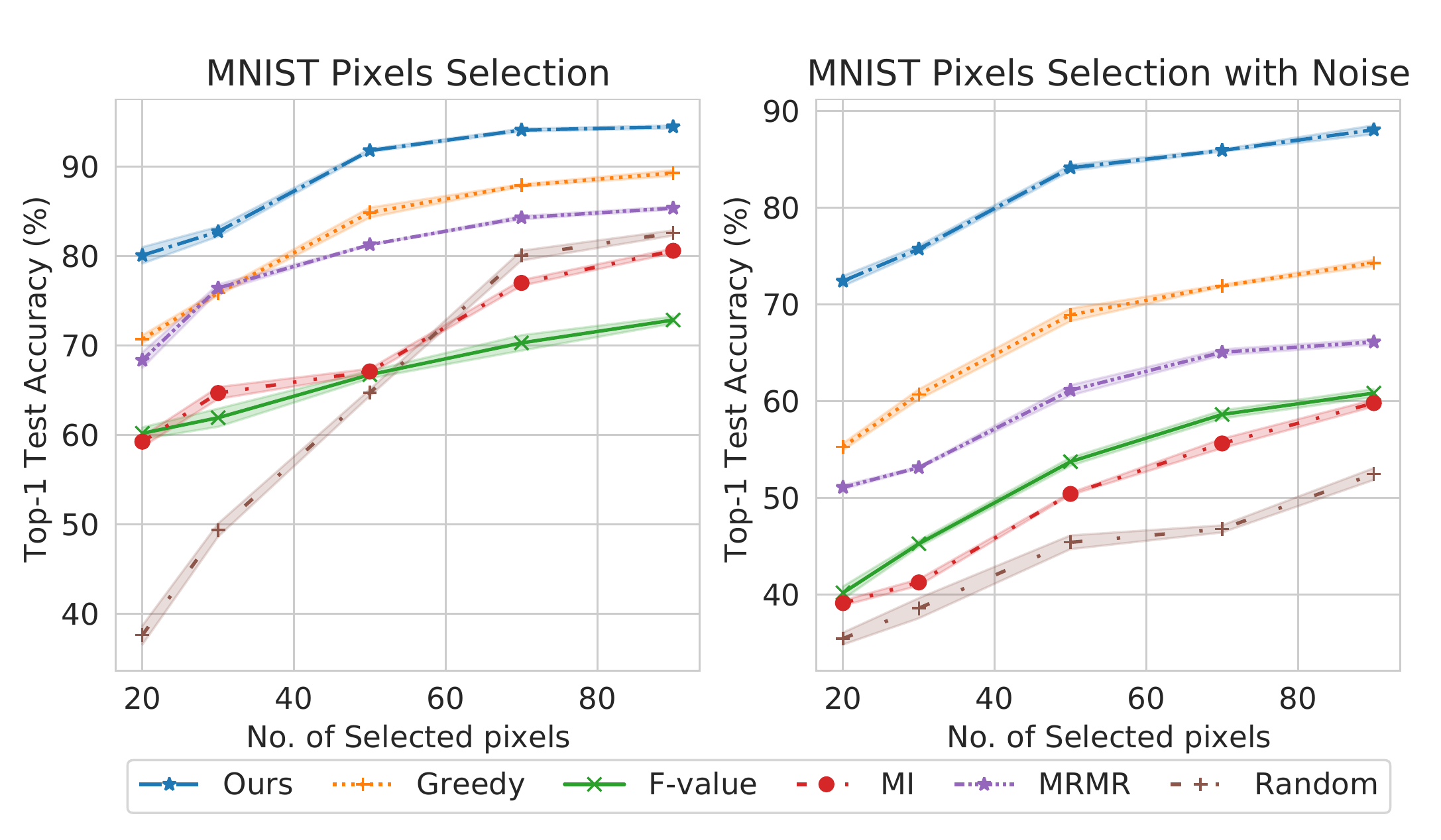}
    \caption{Left: We select a subset of pixel locations, then retain the corresponding pixels of all the training data for training and inference. Right: Pixel selection on dataset with gaussian noise applied to the images. Our method consistently surpasses other baselines by a large margin given different subset size constraints and the advantage is more evident with gaussian noise.}
	\label{fig:pixel_selection}
\end{figure}

We apply our algorithm for feature selection (FS) tasks on the MNIST dataset to further validate its effectiveness. FS can be naturally viewed as a coreset construction problem, where the selected subset of features is the coreset. The adaptation of our method in this case is straightforward: the inner loop is to train the model using the selected pixels of the images, while the outer loop is to update the probability of each pixel based on the loss on all the pixels.

Specifically, we select $n$ pixels locations and the corresponding pixels are used during both training and inference. As shown in Figure \ref{fig:pixel_selection}, for various number of selected features, our method consistently outperforms other baselines including F-score \cite{gunecs2010multi}, mutual information \cite{sulaiman2015feature}, MRMR \cite{radovic2017minimum} and greedy coreset \cite{borsos2020coresets} by a large margin. The selected pixels are demonstrated in Figure \ref{fig:selected_pixels}, which shows that our method is able to retain the most informative features. To further showcase the advantage of our algorithm under more challenging scenario, we impose gaussian noise (with mean set to 0 and stand deviation 2.5) to the image pixels. As demonstrated on the right part of Figure \ref{fig:pixel_selection}, while the performance of most other methods deteriorate significantly, our method still achieves stable performance.

\begin{figure}[ht!]
	\centering
	\includegraphics[width=0.48\textwidth]{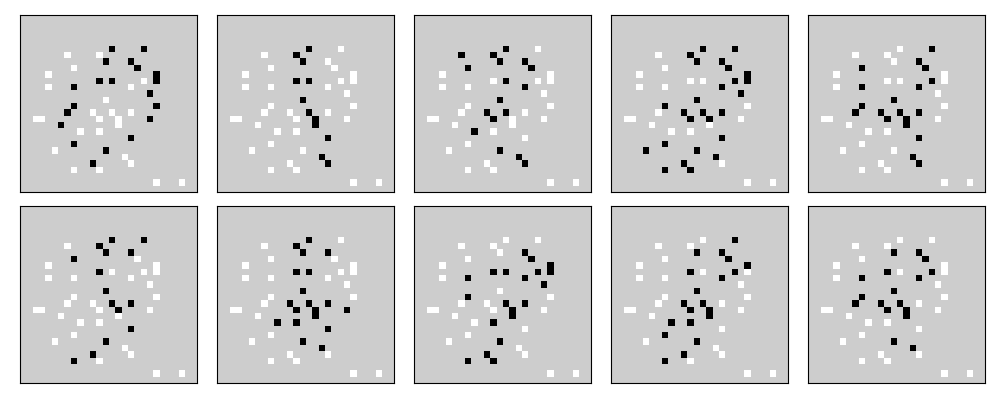}
    \caption{Visualization of selected pixels by our algorithm, where the numbers from 0 to 9 are demonstrated. The white pixels are selected by our algorithm and the black dots are those overlapped with the digits. We can see that the selected pixel locations effectively capture the important information of the images.}
	\label{fig:selected_pixels}
\end{figure}
\section{Ablation Study and Analysis} \label{sec:ablation}
\subsection{Advantage of Global Algorithm Compared with Greedy Counterpart}\label{exp:noise_analysis}
\begin{figure}[ht!]
	\centering
	\includegraphics[width=0.48\textwidth]{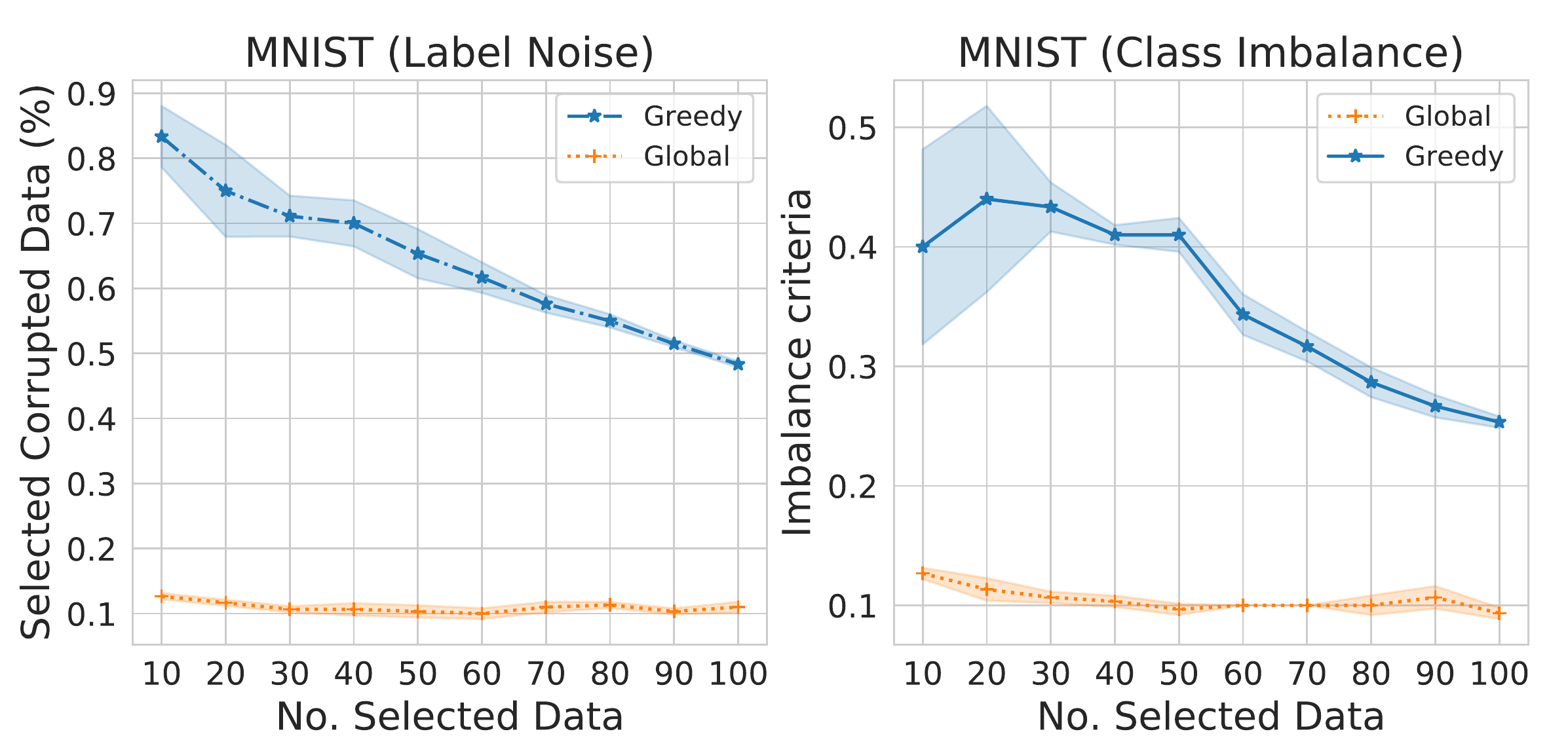}
    \caption{Analysis of how the quality of coreset changes as the coreset size increases. Left: The noise ratio of samples in the coreset. Right: The class imbalance of samples in the coreset, where the y-axis is defined as the difference between the most class and the least class, then devided by the total coreset size. We can see that the poor data quality at the early phase of greedy-based method misguides the search and results in suboptimal coreset.}
	\label{fig:noise_analysis}
\end{figure}

In this experiment, we analyze how the quality of coreset changes as the size increases using MNIST with 1) $90\%$ corrupted labels and 2) imbalance factor of 200. Specifically, we
monitor the change of mislabeling ratio and imbalance criterion in the selected coreset as the size increases. As shown in Figure \ref{fig:noise_analysis}, at the early stage of greedy coreset (blue line), the quality of selected data is much worse due to the lack of global knowledge. Those samples can never be removed and continues to misguide the search, which finally leads to suboptimal coreset. On the other hand, since our method (orange line) samples subsets globally with replacement during the optimization, more information can be learned in this trial-and-error process before arriving at the final coreset (more analysis in Section \ref{thm:convergence}). 
\subsection{Evolution of Coreset during Search}\label{exp:convergence}
\begin{figure}[ht!]
	\centering
    \vspace{-2.3mm}
	\includegraphics[width=0.48\textwidth]{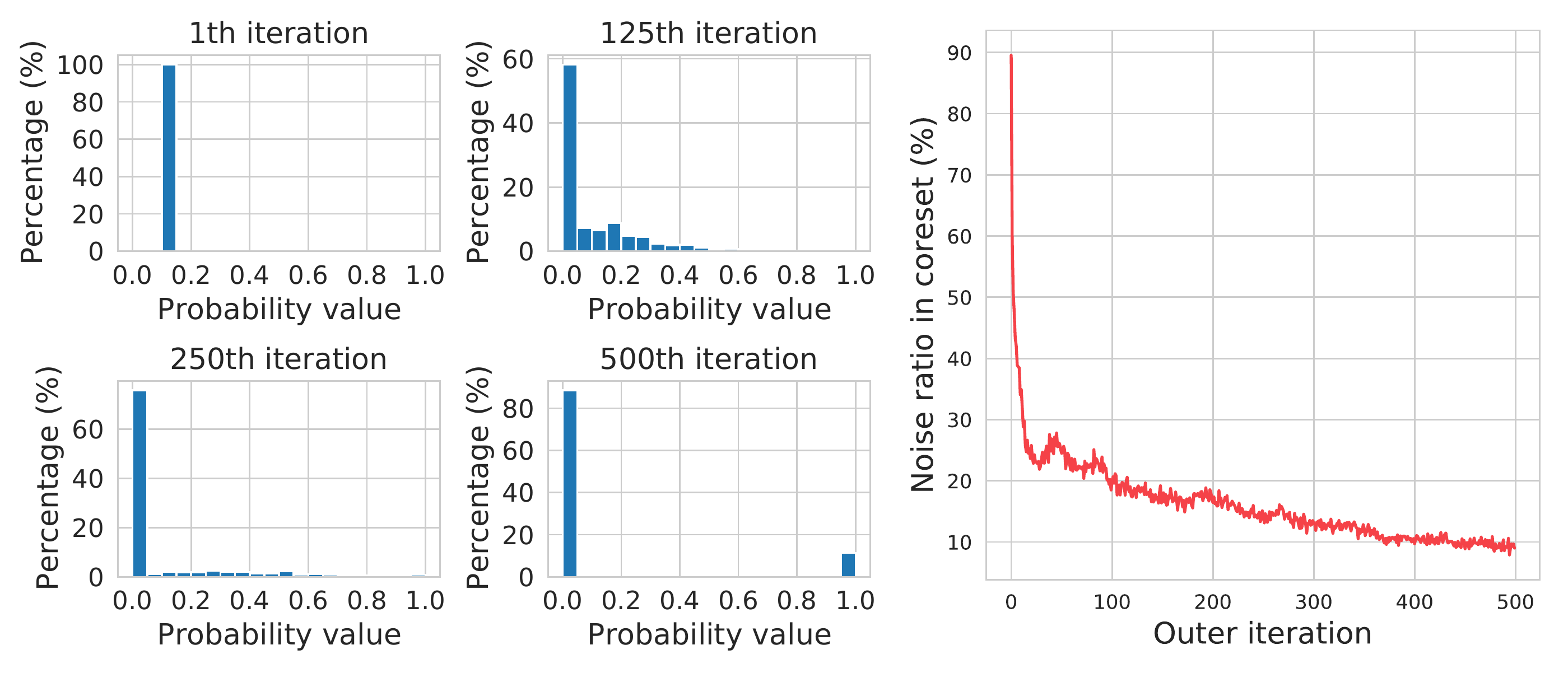}
    \vspace{-5mm}
    \caption{Left: The distribution of probability scores during the search. As the search progresses, most of the score values converge to either 0 or 1, which eventually renders a deterministic coreset with low variance. Right: the noise ratio in the selected coreset as the outer iteration increases. Our method progressively improves the quality of the coreset by learning the global information and updating the probability distribution accordingly.}
    \vspace{-2.5mm}
	\label{fig:probibility_convergence}
\end{figure}

In Figure \ref{fig:probibility_convergence}, we analyze the search process of our method. In this experiment, we select 100 samples from 1000 training data with noise ratio set to 0.9. The left part visualizes how the distribution of probabilities evolve as the search goes on. The initial sample probabilities are evenly distributed and equal to 0.1. As the search goes on, most of the probabilities converge to either 0 or 1, i.e., the uncertainty is gradually reduced to 0, which generates a nearly deterministic sparse mask with low variance. In the right part we show how the noise ratio in the selected coreset evolves. We can observe that the noise ratio continues to decrease, which verifies that our method is able to progressively improve the quality of the coreset by learning the global information.
\subsection{Time Complexity Analysis} \label{exp:time_complexity}
\begin{figure}[ht!]
	\centering
    \vspace{-4.5mm}
	\includegraphics[width=0.48\textwidth]{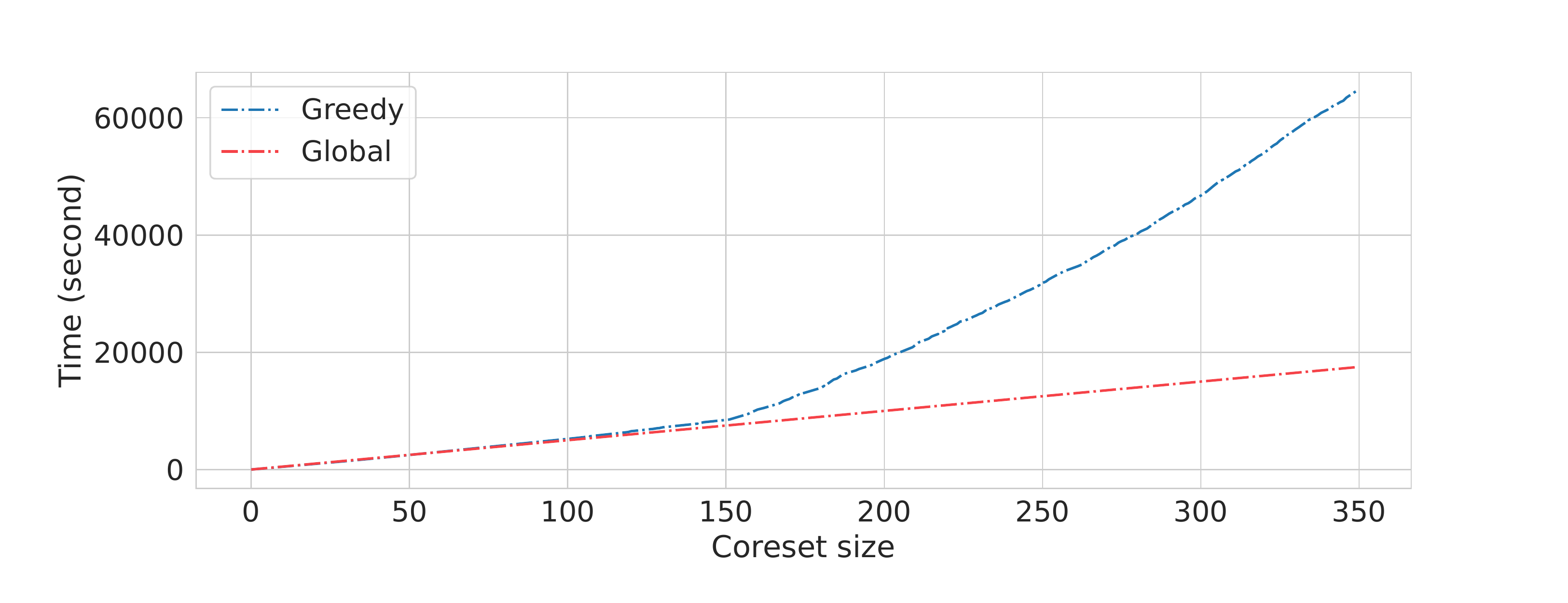}
    \vspace{-8mm}
    \caption{Comparison of time consumption between our method and the greedy counterpart. As the greedy coreset selection method needs to solve a bilevel problem for every newly added sample, the cost increases rapidly with the coreset size.}
    \vspace{-4mm}
	\label{fig:time_process}
\end{figure}

The comparison of time complexities between our method and the greedy counterpart is shown in Figure \ref{fig:time_process}. Since the greedy coreset selection method \cite{borsos2020coresets} needs to solve a bilevel optimization problem for every newly added sample, the time complexity increases rapidly with the coreset size. On the other hand, the time required by our method is not sensitive to the coreset size as the number of outer iterations remains fixed. Furthermore, owing to the efficiency of policy gradients, the update of sample probabilities takes much less time. 
\section{Conclusion}
In this paper, we propose a global coreset selection algorithm based on bilevel optimization and adopt probabilistic reparameterization to continualize the discrete optimization problem. Our method is computationally efficient and achieves promising results even on challenging scenarios with label noise or imbalanced classes. We theoretically prove its convergence and conduct extensive experiments on various tasks to demonstrate its superiority.

\section*{Acknowledgements}
This work is supported by GRF 16201320.

\bibliography{example_paper}
\bibliographystyle{icml2022}

\newpage
\appendix
\onecolumn
\newpage
\appendix
\onecolumn
\icmltitle{Supplementary Materials: Probabilistic Bilevel Coreset Selection}

This appendix can be divided into the following parts:
\begin{itemize}
    \item In Section \ref{sec:proof}, we provide the proof of our property to show the convergence of our method.
    \item In Section \ref{sec:proj}, we give the algorithm for calculating the projection onto our constraint set $\mathcal{C}$.
    \item In Section \ref{sec:experimental-configuration}, we give the detailed configurations of our experiments.
    \item In Section \ref{sec:more-results}, we give more experimental results. 
    \item In Section \ref{sec:future_works}, we present discussions on future works.
\end{itemize}

\section{Proof of Property \ref{thm:convergence}}\label{sec:proof}
At first, we would like to rephrase our theorem into a more formal form below  by adding some assumptions following \cite{pedregosa2016hyperparameter}. 
\begin{theorem}\label{thm:app-convergence}
We assume $\Phi(\bs{s})$  is $L$-smooth, and the policy gradient variance  $\mathbb{E}\|\mathcal{L}_{\mathcal{B}}(\bs{\theta}^*(\bs{m})) \nabla_{\bs{s}}\ln  p(\boldsymbol{m}|\boldsymbol{s}) - \nabla_{\bs{s}}\Phi(\bs{s})\|^2 \leq \sigma^2$. Let $\eta< 1/L$ and we denote the gradient mapping $\mathcal{G}^t$ at $t$-th iteration as  $$\mathcal{G}^t=\frac{1}{\eta}\left(\bs{s}^{t}-\mathcal{P}_{\mathcal{C}}(\bs{s}^{t}-\eta \nabla_{\bs{s}}\Phi(\bs{s}^{t}))\right),$$ then   we have
\begin{align}
     \frac{1}{T}\sum_{t=1}^T \mathbb{E} \|\mathcal{G}^t\|^2 \leq \frac{8-2L\eta}{2-L\eta}\sigma^2, \label{eqn:app-bound}
\end{align}
when $T \rightarrow \infty$.
\end{theorem}
\begin{remark}
We would like to point out  the following things:
\begin{itemize}
    \item We give this theorem just to show that our algorithm works well similarly with  the general projected/proximal stochastic gradient descent algorithms for one level optimization problems, e.g., \cite{ghadimi2016mini}, instead of to show how fast our algorithm can converge. Therefore, we do not consider the techniques, such as variance reduction, to improve the convergence rate of our algorithm, which are out of the scope of this work.
    \item $\sigma^2$ is actually the variance of our gradient estimator PGE with a single mask. It can become smaller with the techniques, such as sample more masks and  use larger batch size in PGE. Therefore, the LHS of Eqn.(\ref{eqn:app-bound}) can converge to a small value. 
    \item Our experimental results show that our algorithm can work well even if we sample only one mask in each iteration. Therefore, in the above theorem, we give the result when only one mask is sampled, to make it consistent with the settings of our experiments. 
\end{itemize}

\end{remark}

Before giving the detailed proof, we need the following lemmas about the properties of the projection operator, which can be found in \cite{ghadimi2016mini}. 
\begin{lemma}[firmly nonexpansive operators] \label{lemma:1}Given a compact convex set $\mathcal{C} \subset \mathbb{R}^d$ and let $\mathcal{P}_{\mathcal{C}}(\cdot)$ be the projection operator on $\mathcal{C}$, then for any $\bs{u}\in \mathbb{R}^d$ and $\bs{v}\in \mathbb{R}^d$, we have 
\begin{align*}
    \|\mathcal{P}_{\mathcal{C}}\left(\bs{u}\right)-\mathcal{P}_{\mathcal{C}}\left(\bs{v}\right)\|^2 \leq \left(\bs{u}-\bs{v}\right)^\top \left(\mathcal{P}_{\mathcal{C}}\left(\bs{u}\right)-\mathcal{P}_{\mathcal{C}}\left(\bs{v}\right)\right). 
\end{align*}
\end{lemma}

\begin{lemma}\label{lemma:2}Given a compact convex set $\mathcal{C} \subset \mathbb{R}^d$ and let $\mathcal{P}_{\mathcal{C}}(\cdot)$ be the projection operator on $\mathcal{C}$, then for any $\bs{c}\in \mathcal{C}$ and $\bs{u}\in \mathbb{R}^d, \bs{v}\in \mathbb{R}^d$, we have 
\begin{align*}
    \|\mathcal{P}_{\mathcal{C}}(\bs{c}+\bs{u})-\mathcal{P}_{\mathcal{C}}(\bs{c}+\bs{v})\| \leq \|\bs{u}-\bs{v}\|. 
\end{align*}
\end{lemma}
\begin{proof} of Theorem \ref{thm:app-convergence}:

In the following, we denote 
\begin{align}
    \bs{g}^t= \mathcal{L}_{\mathcal{B}}(\bs{\theta}^*(\bs{m})) \nabla_{\bs{s}}\ln  p(\boldsymbol{m}|\boldsymbol{s}^t). \nonumber
\end{align}
In our algorithm, we update $\bs{s}$ as
\begin{align}
    \bs{s}^{t+1} = \mathcal{P}_{\mathcal{C}}\left(\bs{s}^t- \eta \bs{g}^t\right). \nonumber
\end{align}
Let the stochastic and deterministic gradient mappings be
\begin{align*}
    \hat{\mathcal{G}}^t =&  \frac{1}{\eta}\left(\bs{s}^{t}-\mathcal{P}_{\mathcal{C}}\left(\bs{s}^t- \eta \bs{g}^t\right)\right) = \frac{1}{\eta}\left(\bs{s}^{t}-\bs{s}^{t+1}\right),\nonumber\\
   \mathcal{G}^t =& \frac{1}{\eta}\left(\bs{s}^{t}-\mathcal{P}_{\mathcal{C}}\left(\bs{s}^t- \eta \nabla\Phi(\bs{s}^t)\right)\right) , \nonumber
\end{align*}
we can have
\begin{align}
    \Phi(\bs{s}^{t+1}) &\leq \Phi(\bs{s}^{t}) + \langle \nabla\Phi(\bs{s}^{t}), \bs{s}^{t+1}-\bs{s}^t \rangle + \frac{L}{2} \|\bs{s}^{t+1}-\bs{s}^t\|^2\nonumber\\
    & = \Phi(\bs{s}^{t}) - \eta \langle \nabla\Phi(\bs{s}^{t}), \hat{\mathcal{G}}^t \rangle + \frac{L\eta^2}{2} \|\hat{\mathcal{G}}^t\|^2\nonumber\\
    & = \Phi(\bs{s}^{t}) -\eta \langle \nabla\Phi(\bs{s}^{t})-\bs{g}^t+\bs{g}^t, \hat{\mathcal{G}}^t \rangle + \frac{L\eta^2}{2} \|\hat{\mathcal{G}}^t\|^2\nonumber\\
     & = \Phi(\bs{s}^{t}) -\eta \langle \bs{g}^t, \hat{\mathcal{G}}^t \rangle + \frac{L\eta^2}{2} \|\hat{\mathcal{G}}^t\|^2 + \eta \langle \delta^t, \hat{\mathcal{G}}^t \rangle (\mbox{ here }\delta^t =\bs{g}^t - \nabla\Phi(\bs{s}^{t}))\nonumber\\
     & \leq  \Phi(\bs{s}^{t}) -\eta \|\hat{\mathcal{G}}^t\|^2 + \frac{L\eta^2}{2} \|\hat{\mathcal{G}}^t\|^2 + \eta \langle \delta^t, \hat{\mathcal{G}}^t \rangle ~~(\mbox{Lemma} \ref{lemma:1})\nonumber\\
     & \leq  \Phi(\bs{s}^{t}) -(\eta-\frac{L\eta^2}{2}) \|\hat{\mathcal{G}}^t\|^2 + \eta \langle \delta^t, \hat{\mathcal{G}}^t \rangle \nonumber\\
     &=  \Phi(\bs{s}^{t}) -(\eta-\frac{L\eta^2}{2}) \|\hat{\mathcal{G}}^t\|^2 + \eta \langle \delta^t, \mathcal{G}^t \rangle + \eta \langle \delta^t, \hat{\mathcal{G}}^t- \mathcal{G}^t \rangle  \nonumber\\
     & \leq \Phi(\bs{s}^{t}) -(\eta-\frac{L\eta^2}{2}) \|\hat{\mathcal{G}}^t\|^2 + \eta \langle \delta^t, \mathcal{G}^t \rangle + \eta \| \delta^t\|\| \hat{\mathcal{G}}^t- \mathcal{G}^t \|  \nonumber\\
      &\leq  \Phi(\bs{s}^{t}) -(\eta-\frac{L\eta^2}{2}) \|\hat{\mathcal{G}}^t\|^2 + \eta \langle \delta^t, \mathcal{G}^t \rangle + \eta \| \delta^t\|^2. ~~(\mbox{Lemma} \ref{lemma:2})\nonumber
\end{align}
Therefore, we can get
\begin{align*}
    (\eta-\frac{L\eta^2}{2}) \|\hat{\mathcal{G}}^t\|^2 \leq \Phi(\bs{s}^{t}) - \Phi(\bs{s}^{t+1}) + \eta \langle \delta^t, \mathcal{G}^t \rangle + \eta \| \delta^t\|^2.
\end{align*}
Thus, we can obtain
\begin{align}
    \sum_{t=1}^T(\eta-\frac{L\eta^2}{2}) \|\hat{\mathcal{G}}^t\|^2 \leq \Phi(\bs{s}^{1}) - \Phi(\bs{s}^{T+1}) +\eta\sum_{t=1}^T\left( \langle \delta^t, \mathcal{G}^t \rangle +  \| \delta^t\|^2\right). \label{eqn:app-sum-bound}
\end{align}
Now, we turn to analyze the item $ \langle \delta^t, \mathcal{G}^t \rangle$ as follows:
\begin{align}
    \mathbb{E}\langle \delta^t, \mathcal{G}^t \rangle &= \mathbb{E}_{\bs{s}^t}\mathbb{E}_{\cdot|\bs{s}^t}\left(\langle  \bs{g}^t - \nabla\Phi(\bs{s}^{t}), \mathcal{G}^t \rangle|\bs{s}^t\right) =0, \label{eqn:app-delta-g}
\end{align}

For $\|\delta^t\|^2$,  we have 
\begin{align}
    \mathbb{E}\|\delta^t\|^2 = \mathbb{E} \|\bs{g}^t - \nabla\Phi(\bs{s}^{t})\|^2 \leq \sigma^2.\label{eqn:app-delta}
\end{align}

 Combining the inequalities (\ref{eqn:app-delta-g}), (\ref{eqn:app-sum-bound}) and (\ref{eqn:app-delta}), we can have 
\begin{align}
    \frac{1}{T}\sum_{t=1}^T \mathbb{E}\|\hat{\mathcal{G}}^t\|^2 \leq \frac{\Phi(\bs{s}^1) - \Phi^*}{(1-L\eta/2)T} + \frac{\sigma^2}{1-L\eta/2}. \label{eqn:app-hg}
\end{align}

Finally, we bound $\mathbb{E}\|\mathcal{G}^t\|^2$ as follows:
\begin{align}
    \mathbb{E} \|\mathcal{G}^t\|^2
    & \leq  2  \mathbb{E} \|\hat{\mathcal{G}}^t\|^2   +2\mathbb{E}\| \bs{g}^t-\nabla \Phi(\bs{s}^t)\|^2 \\
    & \leq  2  \mathbb{E} \|\hat{\mathcal{G}}^t\|^2  +2\sigma^2. \label{eqn:app-gt}
\end{align}

Combine inequalities (\ref{eqn:app-gt}) and (\ref{eqn:app-hg}), when $T\rightarrow \infty$, we can obtain 
\begin{align*}
     \frac{1}{T}\sum_{t=1}^T \mathbb{E} \|\mathcal{G}^t\|^2 &\leq \frac{2}{1-L\eta/2}\left(\frac{\Phi(\bs{s}^1) - \Phi^*}{T}  + (2-L\eta/2)\sigma^2\right)\rightarrow \frac{8-2L\eta}{2-L\eta}\sigma^2
\end{align*}

\end{proof}

\section{Project Calculation} \label{sec:proj}

The projection from $\bs{s}$ to $\mathcal{C}$ can be calculated by:


\begin{algorithm}[htb!]
\caption{Projection from $\bs{z}$ to $\mathcal{C}$}
\label{alg:proj}
\begin{algorithmic}[1]
\REQUIRE a vector $\bs{z}$.
\STATE Solve $v_1$ from $\bs{1}^\top[\min (1, \max(0, \bs{z}-v_{1}^{*}\mathbf{1}))] - K = 0.$ \\
\STATE $v_{2}^{*} \leftarrow \max(0, v_{1}^{*})$. \\
\STATE $\bs{s} \leftarrow \min (1, \max(0, \bs{z}-v_{2}^{*}\mathbf{1})).$ \\
\OUTPUT $\bs{s}$
\end{algorithmic}
\end{algorithm}
\begin{proof}
The projection from $\bs{z}$ to set $\mathcal{C}$ can be formulated in the following optimization problem:
\begin{align}
    &\min_{\bs{s}\in \mathbb{R}^n} \frac{1}{2}\|\bs{s}-\bs{z}\|^2,\nonumber\\
    s.t. & \mathbf{1}\top \bs{s}  \leq K \mbox{ and } 0\leq \bs{s}_i \leq 1.\nonumber
\end{align}
Then we solve the problem with Lagrangian multiplier method.
\begin{align}
    L(\bs{s},v) &= \frac{1}{2}\|\bs{s}-\bs{z}\|^2 + v(\mathbf{1}^\top \bs{s} -K)\\
    &=\frac{1}{2}\|\bs{s}-(\bs{z}-v\mathbf{1})\|^2 + v (\mathbf{1}^\top \bs{z}-K) -\frac{n}{2}v^2.
\end{align}
with $v \geq 0 \mbox{ and } 0 \leq \bs{s}_i \leq 1$.
Minimize the problem with respect to $\bs{s}$, we have 
\begin{align}
    \tilde{\bs{s}} = \mathbf{1}_{\bs{z}-v\mathbf{1}\geq 1} + (\bs{z}-v\mathbf{1})_{1>\bs{z}-v\mathbf{1}>0}
\end{align}
Then we have
\begin{align}
g(v)=&L(\tilde{\bs{s}},v) \nonumber\\
   =& \frac{1}{2}\|[\bs{z}-v\mathbf{1}]_{-} + [\bs{z}-(v+1)\mathbf{1}]_{+}\|^2 \nonumber \\
   &+ v (\mathbf{1}^\top \bs{z}-s) -\frac{n}{2}v^2 \nonumber  \\
=&\frac{1}{2}\|[\bs{z}-v\mathbf{1}]_{-}\|^2 +\frac{1}{2}\|[\bs{z}-(v+1)\mathbf{1}]_{+}\|^2\nonumber \\
&+ v (\mathbf{1}^\top \bs{z}-s) -\frac{n}{2}v^2,  v\geq 0. \nonumber \\
g'(v)=& \mathbf{1}^\top [v\mathbf{1}-\bs{z}]_{+} +\mathbf{1}^{\top} [(v+1)\mathbf{1}-\bs{z}]_{-}\nonumber \\
&+(1^T\bs{z}-s)-nv \nonumber \\
=&\mathbf{1}^\top\min (1, \max(0, \bs{z}-v\mathbf{1})) - K,v\geq 0.\nonumber
\end{align}
It is easy to verify that $g'(v)$ is a monotone decreasing function with respect to $v$ and we can use a bisection method solve the equation $g'(v) = 0$ with solution $v^*_1$. Then we get that $g(v)$ increases in the range of $(-\infty, v^*_1$] and decreases in the range of $[v^*_1, +\infty)$. The maximum of g(v) is achieved at 0 if $v^*_1 \leq 0$ and $v^*_1$ if $v^*_1 >0$. Then we set $v^*_2 = max(0, v^*_1)$. Finally we have 
\begin{align}
    \bs{s}^* =& \mathbf{1}_{\bs{z}-v_2^*\mathbf{1}\geq 1} + (\bs{z}-v_{2}^*\mathbf{1})_{1>\bs{z}-v_2^*\mathbf{1}>0}\\ 
    =&\min (1, \max(0, \bs{z}-v_{2}^{*}\mathbf{1})).
\end{align}
\end{proof}

\section{Experiment Details} \label{sec:experimental-configuration}
We use the following hyper-parameters during optimization for our experiments. For the inner-loop, the model is trained for 100 epochs using SGD with learning rate of 0.1 and momentum of 0.9. For the outer-loop, the probabilities $s$ are optimized by adam with learning rate of 2.5 and cosine scheduler. The outer-loop is updated for 500-2000 times. Note that gumbel softmax can also be used as an alternative for PGE when calculating the hyper-gradient, which sometimes demonstrates better stability. In implemetation, we combine gumbel softmax and PGE to achieve a balance between efficiency and accuracy. 

\textbf{Label Noise and Class Imbalance}
For pairwise noise, the label of a particular class has a probability $p$ to be flipped to the adjacent class; for symmetric noise, a class has a probability $\frac{p}{n-1}$ to be changed to any other $n-1$ classes. For the class imbalance experiment, we adopt similar setting as in \cite{cui2019class}. Specifically, the number of training samples per class is exponentially reduced according to the function $n_i'=n_i\sigma^i$, where $i$ is the class index. We define dataset imbalance factor as $\frac{n_{max}}{n_{min}}$, where $n_{max}$ and $n_{min}$ are the number of samples in the largest and smallest classes, respectively.

\textbf{Data Summarization} We follow \cite{borsos2020coresets} and use a convolutional neural network stacked with two blocks of convolution, dropout, max-pooling and ReLU activation for MNIST, and use ResNet18\cite{he2016deep} for the CIFAR10 experiment. For baselines K-center clustering using the embedding from last layer \cite{sener2017active}, iCaRL's selection \cite{rebuffi2017icarl} and Hardest sampling \cite{aljundi2019task}, which depend on the model embedding, we pretrain a feature extractor using 1000 uniformly sampled data. 

\textbf{Continual learning} To make a fair comparison, we follow \cite{borsos2020coresets} to keep 1000 subsamples for each task other than SplitCIFAR100 (where all data is used). For PermMNIST, we use a fully-connected network with two hidden layers of 100 neurons, followed by ReLU activation and dropout with probability 0.2, the memory size is set to 100. For SplitMNIST, we use the same CNN architecture as in data summerization task, the memory size is 500. For SplitCIFAR10, we adopt ResNet18 similar to Section \ref{sec:coreset}, the memory is set to 200.
As for SplitCIFAR100, all the training data is used and the memory size is set to 2000. For experiments with label noise and class imbalance, a balanced held-out validation dataset of size 100 is used.

\textbf{Streaming} We follow \cite{borsos2020coresets} to modify PermMNIST and SplitMNIST used in continual learning by concatenating all the tasks and stream the data in subsets of 125. The memory size in this setting is set to 100 and separated into 10 slots. Merge reduced introduced in \cite{borsos2020coresets} is also adopted for fair comparison.

In feature selection experiments, we uniformly sample 1000 training data, while all the testing data are used to evaluate the trained model. All the experiments are repeated with 5 different random seeds.

\section{More Experiments}\label{sec:more-results}
\subsection{Comparison with Training on Entire Dataset under Label Noise and Class Imbalance Settings}
\begin{figure}[ht!]
	\centering
    \vspace{-2mm}
	\includegraphics[width=0.8\textwidth]{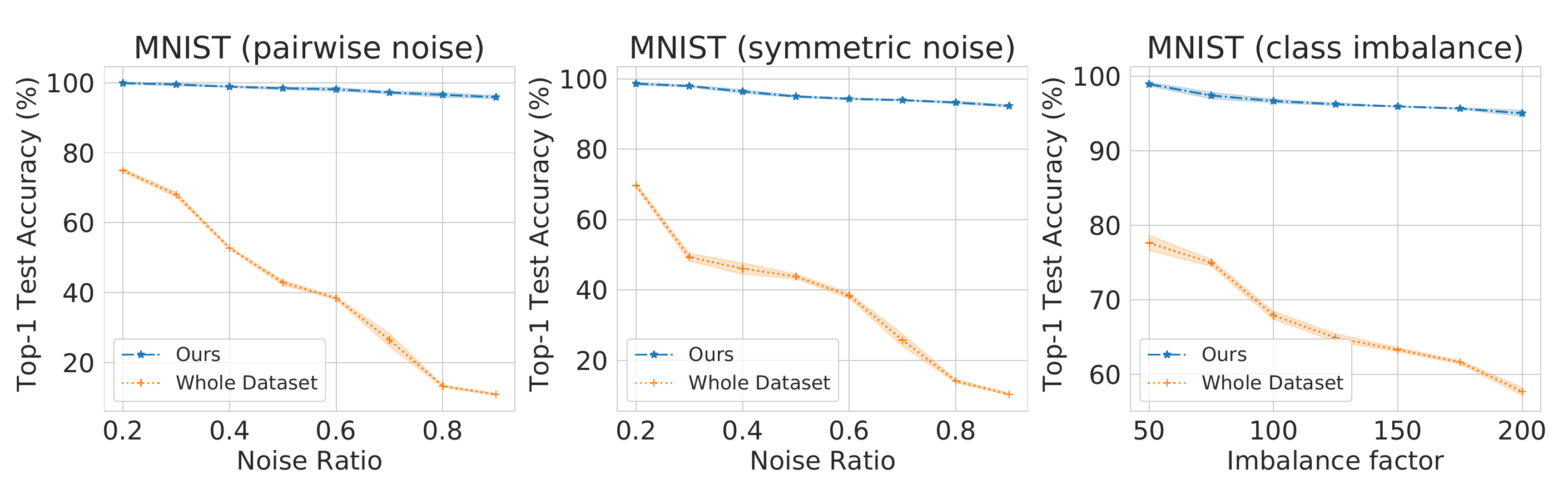}
    \vspace{-3mm}
    \caption{We compare the performance of the trained model on the coreset selected by our method with that trained on the entire dataset, where the dataset contains label noise and class imbalance. We can see that the training with the selected coreset surpasses the entire dataset by a large margin, which is because the coreset can effectively remove the label noise and automatically balance the data in different classes.}
    \vspace{-3mm}
	\label{fig:ours_vs_whole}
\end{figure}
As mentioned in Section \ref{sec:introduction} of the main paper, training on the coreset may sometimes even achieve better performance than training on the entire dataset. We verify this via conductin experiment on dataset with label noise and class imbalance. We can see that training the model using the entire dataset in these cases lead to failure due to the poor data quality. On the other hand, our coreset selection method has the effect of removing the label noise and automatically balancing the data in each class.

\subsection{Transferability of found coresets with various sizes among different networks}
\begin{table}[htb!]
\caption{Transferability of found coresets with various sizes among different networks. 
}
\label{tab:transfer}
\begin{center}
\begin{tabular}
{
l m{1cm}<{\centering} m{1cm}<{\centering} m{1cm}<{\centering} m{1cm}<{\centering}
}
\toprule
&\multicolumn{2}{c}{Net1}&\multicolumn{2}{c}{Net2}\\Dataset / Size&
Random&Coreset&Random&Coreset\\\cmidrule(l){2-3}\cmidrule(l){4-5}
MNIST (100)&87.3 &95.01& 70.90 &75.60\\
CIFAR10 (4000)&63.50 &78.35& 67.57&79.05\\
\bottomrule
\end{tabular}
\end{center}
\vskip -0.1in
\end{table}

We use the coreset selected by one network to train different networks. In Table \ref{tab:transfer}, we show the results of the following experiment: The coresets is searched by Net1, which is then used to train both Net1 itself and Net2. The found coreset outperforms uniform sampling for both 
networks, which verifies the transferability of the found coreset. For CIFAR10, Net1 is ResNet18 and Net2 is ResNet32; for MNIST, Net1 is Convnet and Net2 is MLP.

\section{Future Directions} \label{sec:future_works}
Our coreset selection also performs rather faster than \citet{borsos2020coresets}, it stills faces computational difficulties when applied to larger datasets like ImageNet-1K. One possible solution is to employ model sparsity to speed-up the training process or make the trained model smaller \cite{Shao_2019_CVPR, zhou2021efficient, zhou2021effective}.

\end{document}